\documentclass[letterpaper, 10 pt, conference]{ieeeconf}  

\IEEEoverridecommandlockouts                              

\usepackage{graphics} 
\usepackage{cite}
\usepackage{times} 
\usepackage{amsmath} 
\usepackage{amssymb}  
\usepackage{xcolor}  
\usepackage[ruled, vlined, linesnumbered]{algorithm2e}  
\SetArgSty{textup}  
\usepackage[font=small]{subcaption}  
\usepackage{hyperref}  
\usepackage[disable]{todonotes}
\usepackage{eso-pic}

\newcommand{\Alg}[1]{{\fontfamily{cmtt}\selectfont#1}}

\newtheorem{theorem}{Theorem}
\newtheorem{problem}[theorem]{Problem}

\title{\LARGE \bf
Conflict-based Search for Multi-Robot Motion Planning with Kinodynamic Constraints
}

\author{Justin Kottinger$^{1}$, Shaull Almagor$^{2}$, and Morteza Lahijanian$^{1}$
\thanks{$^{1}$Aerospace Engineering Sciences, University of Colorado Boulder, USA
        {\tt\small \{firstname.lastname\}@colorado.edu}}%
\thanks{$^{2}$The Henry and Marilyn Taub Faculty of Computer Science, Technion, Israel
        {\tt\small shaull@cs.technion.ac.il}%
}
}

\begin{document}
\AddToShipoutPictureBG*{%
  \AtPageUpperLeft{%
    \hspace{16.5cm}%
    \raisebox{-1.5cm}{%
      \makebox[0pt][r]{To appear in the International Conference on Intelligent Robots and Systems (IROS), October 2022.}}}}

\maketitle
\thispagestyle{empty}
\pagestyle{empty}

\begin{abstract}
\emph{Multi-robot motion planning} (MRMP) is the fundamental problem of finding non-colliding trajectories for multiple robots acting in an environment, under kinodynamic constraints.
Due to its complexity, existing algorithms either utilize simplifying assumptions or are incomplete.
This work introduces \emph{kinodynamic conflict-based search} (\Alg{K-CBS}), a decentralized (decoupled) MRMP algorithm that is general, scalable, and probabilistically complete. The algorithm takes inspiration from successful solutions to the discrete analogue of MRMP over finite graphs, known as \emph{multi-agent path finding} (MAPF).
Specifically, we adapt ideas from \emph{conflict-based search} (CBS)--a popular decentralized MAPF algorithm--to the MRMP setting.  
The novelty in this adaptation is that we work directly in the continuous domain, without the need for discretization. In particular, the kinodynamic constraints are treated natively.
\Alg{K-CBS} plans for each robot individually using a low-level planner and and grows a conflict tree to resolve collisions between robots by defining constraints for individual robots.  The low-level planner can be any sampling-based, tree-search algorithm for kinodynamic robots, thus lifting existing planners for single robots to the multi-robot settings.   
We show that \Alg{K-CBS} inherits the (probabilistic) completeness of the low-level planner.  We illustrate the generality and performance of \Alg{K-CBS} in several case studies and benchmarks.


\end{abstract}
\section{Introduction}
\label{sec:intro}


\textit{Multi-robot motion planning} (MRMP) is a fundamental challenge in robotics and artificial intelligence. The goal of MRMP is to compute dynamically-feasible trajectories for multiple robots to reach their respective goal regions such that, when executed simultaneously, the robots do not collide with obstacles nor with each other. MRMP has vast applications, from  warehouse robotics and delivery to planetary exploration and search-and-rescue.  
The difficulty of MRMP lies in the size of the planning space, which grows exponentially with the number of robots.  In addition, there are constraints posed by the robots kinodynamics that need to be respected, which by itself is a difficult problem even for a single robot -- \textbf{NP-hard} for simple cases of a point robot with Newtonian dynamics \cite{DXPR93nphard} and curvature-constrained planar Dubins car (linear dynamics) \cite{kirkpatrick2011hardness}, and for nonlinear robots, the complexity is unknown (likely undecidable).
For this reason, either simplifying assumptions are posed on the problem (e.g., removal of the dynamic constraints) or incomplete algorithms are accepted for the purposes of scalability.
In this work, without any simplifying assumptions, we aim to develop a scalable kinodynamic planner for MRMP with (probabilistic) completeness guarantees.

MRMP has been extensively studied in the robotics community.  
The approaches are generally classified into two categories: \emph{centralized} (coupled) and \emph{decentralized} (decoupled). 
The centralized methods (e.g., \cite{wagner2015subdimensional,shome2020drrt})
operate over the composed state space of all the robots.   
The advantage of these approaches is that they allow the use of the rich library of existing single-robot motion planners (e.g., RRT \cite{Lavalle98rapidly-exploringrandom}
and KPIECE \cite{sucan2011sampling}) and automatically inherit their completeness and optimality.
The disadvantage is that they are not scalable.  
In contrast, 
decentralized approaches divide the problem into several sub-problems and solve each separately \cite{gammell2014bit,sanchez2002using,van2005prioritized,tang2018complete}.
While successful in scalability, 
existing decentralized MRMP algorithms are mostly incomplete, i.e., do not guarantee to find a solution if one exists.

The discrete analogue of MRMP, dubbed \emph{multi-agent path finding} (MAPF), has received much attention in the AI and discrete planning community. There, the agents travel along edges in a discrete graph (rather than a continuous space).
The main focus in MAPF is scalability, and as a result,  advanced algorithms have been developed that are not only scalable, but also have completeness and optimality guarantees (see \cite{stern2019mapf} and references therein). 
A popular decentralized algorithm is \emph{conflict-based search} (CBS) \cite{SHARON201540}, which plans for each agent separately using low-level search (e.g., $A^*$). It then identifies collisions in the proposed plans and re-plans for colliding agents after placing path-constraints that resolve the collision. Thus, CBS grows a ``conflict tree'', where each node represents a (possibly colliding) plan.

Unfortunately, such MAPF algorithms cannot be utilized to solve the MRMP problem in a straightforward manner, since MAPF abstracts the continuous workspace into a finite graph (e.g., a grid world),
and ignores the shape of the robots.  It also bypasses all dynamic constraints by assuming every edge of the graph is realizable with a known time duration. 

Several works have focused on adapting MAPF algorithms to handle the challenges of MRMP, e.g.,
\cite{solis2021representation, honig2018trajectory, Le:ICAPS:2017, Le:RAL:2019, WEN2022103997}.  
Work \cite{solis2021representation} combines continuous \emph{probabilistic roadmaps} (PRM) planner \cite{kavraki1996probabilistic}
with CBS to solve the robotics version of MAPF by ignoring the kinodynamic constraints.  Work \cite{honig2018trajectory}
accounts for dynamics and
uses CBS with optimization techniques for a swarm of quadrotors.  
Specifically, it first performs a discretization of the workspace (grid), and uses CBS to compute discrete paths.  Then, it uses an optimization formulation to generate controllers for the quadrotors to follow the discrete paths.  To guarantee correctness, it assumes sphere-shaped quadrotors with linear dynamics.
While performing well for specific robots, those approaches do not generalize to robots with nonlinear dynamics.


To deal with general nonlinear dynamics, a powerful approach is sampling-based, tree-based motion planning.  To this end, recent work 
\cite{Le:ICAPS:2017,Le:RAL:2019} combine discrete search with a sampling-based tree planner to solve the MRMP problem.  The algorithm first discretizes the composed configuration space of all the robots using PRM, and then uses a discrete planner to find a candidate high-level plan (guide). 
Then, a sampling-based motion tree is grown in the composed space to follow the sequence of regions in the guide. If unsuccessful, a different guide is computed. This process repeats until a valid trajectory is found.  
The work uses centralized discrete planner at the high-level and grows the motion tree in the composed configuration space. Hence, it is probabilistically complete (i.e., the probability of finding a solution tends to one as the planning time tends to infinity, if a solution exists), but it may be slow. To speed up the planner, a feedback PID controller is assumed to drive each robot from one discrete region in the guide to the next.  While fast planning times can be achieved using this method, the choice for the size of PRM (discretization) is unclear, and the assumption on the controller is limiting, as they do not always exist.  

In this work, we introduce a decentralized, scalable MRMP planner that treats the kinodynamic constraints of the robots natively with probabilistic completeness guarantees.  Unlike related work that directly adopt MAPF algorithms on a discretized version of the problem, we incorporate the ideas that make CBS successful into the continuous domain using a sampling-based method. 
Our algorithm, dubbed \emph{Kinodynamic CBS} (\Alg{K-CBS}), like CBS, uses a low-level search and maintains a conflict tree.  The low-level search can be any (kinodynamic) sampling-based tree planner (e.g., RRT and KPIECE).  In each iteration of the algorithm, a plan is computed for an individual robot given a set of constraints (obstacles).  Then, collisions between the robot trajectories is checked.  If one exists, time-dependent obstacles are defined as constraints in the constraint tree, and a new planning query is specified accordingly.  
To ensure probabilistic completeness, we introduce a \emph{merge} method, by which we merge robots whose plans often conflict, into a single meta-robot.


The main contribution of this work is a decentralized, probabilistically-complete MRMP algorithm that is capable of generating kinodynamically feasible plans efficiently.  
Our planner, \Alg{K-CBS}, naturally adapts CBS from MAPF to MRMP. 
This lifts every off-the-shelf (kinodynamic) sampling-based, tree-based planner to the multi-robot setting and removes all the limitations (assumptions) of state-of-the-art MRMP planners. Specifically, \Alg{K-CBS} operates completely in the continuous state space of the agents, hence, it does not require discretization nor a feedback-control design. It easily works with arbitrary, possibly heterogeneous, nonlinear dynamical models, and is capable of solving very complex MRMP instances efficiently. 
We empirically show the efficacy of \Alg{K-CBS} in many case studies, highlighting our algorithm's generality and performance improvements.

\section{Problem Formulation}
\label{sec:Problem}

Consider $k\in \mathbb{N}$ robotic systems,
in a shared, bounded workspace $W\subseteq \mathbb{R}^d$, $d \in \{2,3\}$, 
which includes a finite set of obstacles $\mathcal{O}$, where every obstacle $o\in \mathcal{O}$ is a closed subset of $W$. Furthermore, every robot $i\in\{1,2,\ldots,k\}$ has a (rigid) body $\mathcal{S}_i$, and its motion is subjected to the following dynamic constraint:
\begin{equation}
    \dot{\mathbf{x}}_i=f_i(\mathbf{x}_i, \mathbf{u}_i), \quad \mathbf{x}_i \in X_i \subseteq \mathbb{R}^{n_i} \quad \mathbf{u}_i=U_i\subseteq\mathbb{R}^{m_i},
    \label{eqn:dyn}
\end{equation}
where $n_i \geq d$, $X_i$ and $U_i$ are robot $i$'s state and input spaces, respectively, and $f_i:X_i\times U_i\rightarrow X_i$ is a continuous and possibly nonlinear function (vector field). With an abuse of notation, we denote by $\mathcal{S}_i(\mathbf{x}_i) \subseteq W$ the set of points that the robot $i$'s body occupies in the workspace when it is at state $\mathbf{x}_i \in X_i$.

Given a time interval $[t_0, t_{fi}]$, where $t_0,t_{fi}\in \mathbb{R}_{\geq 0}$ and $t_0 \leq t_{fi}$, a controller $\mathbf{u}_i:[t_0, t_{fi}]\rightarrow U_i$, 
and initial state $\mathbf{x}_{i,0}\in X_i$ for robot $i$, function $f_i$ can be integrated up to time $t_{fi}$ to form a 
\emph{trajectory} $T_i:\mathbb{R}_{\geq 0}\rightarrow X_i$ for robot $i$, where $T_i(t_0) = \mathbf{x}_{i,0}$.  
We assume all the robots' trajectories start at time $t_0$.
Further, we assume that, once robot $i$ reaches the end of its trajectory $T_i(t_{fi})$ at time $t_{fi}$, it remains there, i.e., $T_i(t') = T_i(t_{fi})$ for all $t' \geq t_{fi}$.  
Given trajectories  $T_1, T_2, \ldots, T_k$, denote by $t_f = \max \{t_{f1}, t_{f2}, \ldots t_{fk}\}$ the longest time duration of the trajectories.  We say that the trajectories are \emph{collision-free} if the following conditions hold for all $t \in [t_0, t_f]$: 
\begin{enumerate}
    \item $\forall 1\le i\le k$,  $\forall o \in \mathcal{O}$, \quad $\mathcal{S}_i(T_i(t)) \cap o = \emptyset$  \label{cond:obsCollision}
    \item $\forall 1\le i<j\le k$, \qquad \;\;  $\mathcal{S}_i(T_i(t)) \cap \mathcal{S}_j(T_j(t)) = \emptyset$ \label{cond:robotColliision}
\end{enumerate}

\noindent
Condition~\ref{cond:obsCollision} indicates no collision with obstacles, and Condition~\ref{cond:robotColliision} indicates no robot-robot collision.

Given a goal region for every robot, the objective in MRMP is to find a feasible trajectory for every robot to reach its goal without colliding with obstacles nor other robots:
\begin{problem}[MRMP]
\label{prob}

Given $k$ robots in workspace $W$ with dynamics as in \eqref{eqn:dyn}, initial states $\mathbf{x}_{i,0} \in X_i$ and goal regions $X_i^G \subseteq X_i$ for all $1\le i\le k$, 
find a controller $\mathbf{u}_i:[t_0, t_{fi}]\rightarrow U_i$ for every $1\le i\le k$ such that the induced trajectories $T_1, \ldots, T_k$ are collision-free, and $T_i$ takes agent $i$ from $T_i(t_0)=\mathbf{x}_{i,0}$ to $T_i(t_{fi})\in X_i^G$ for every $1\le i\le k$.
\end{problem}

\section{Methodology}
\label{sec:Methods}
We now present our solution to Problem~\ref{prob}, which takes inspiration from CBS. To this end, we first describe CBS for MAPF and then introduce the necessary extensions to solve the MRMP problem. Finally, we introduce our algorithm, \emph{Kinodynamic} CBS (\Alg{K-CBS}). 



\begin{figure*}[t]
    \centering
    \begin{subfigure}{0.24\textwidth}
        \centering
        \includegraphics[width=\textwidth]{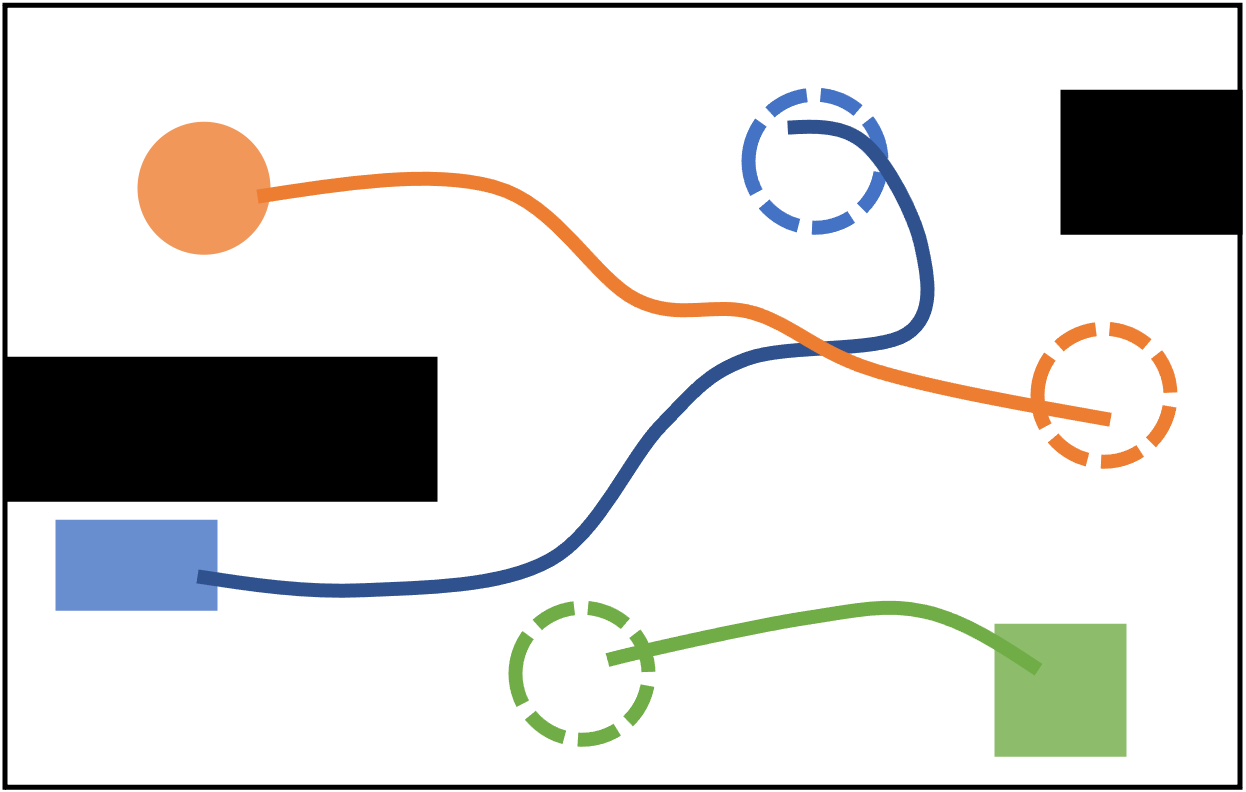}%
        \caption{Proposed plan}%
        \label{fig:demo_initial}%
    \end{subfigure}\hfill
    \begin{subfigure}{0.24\textwidth}
        \centering
        \includegraphics[width=\textwidth]{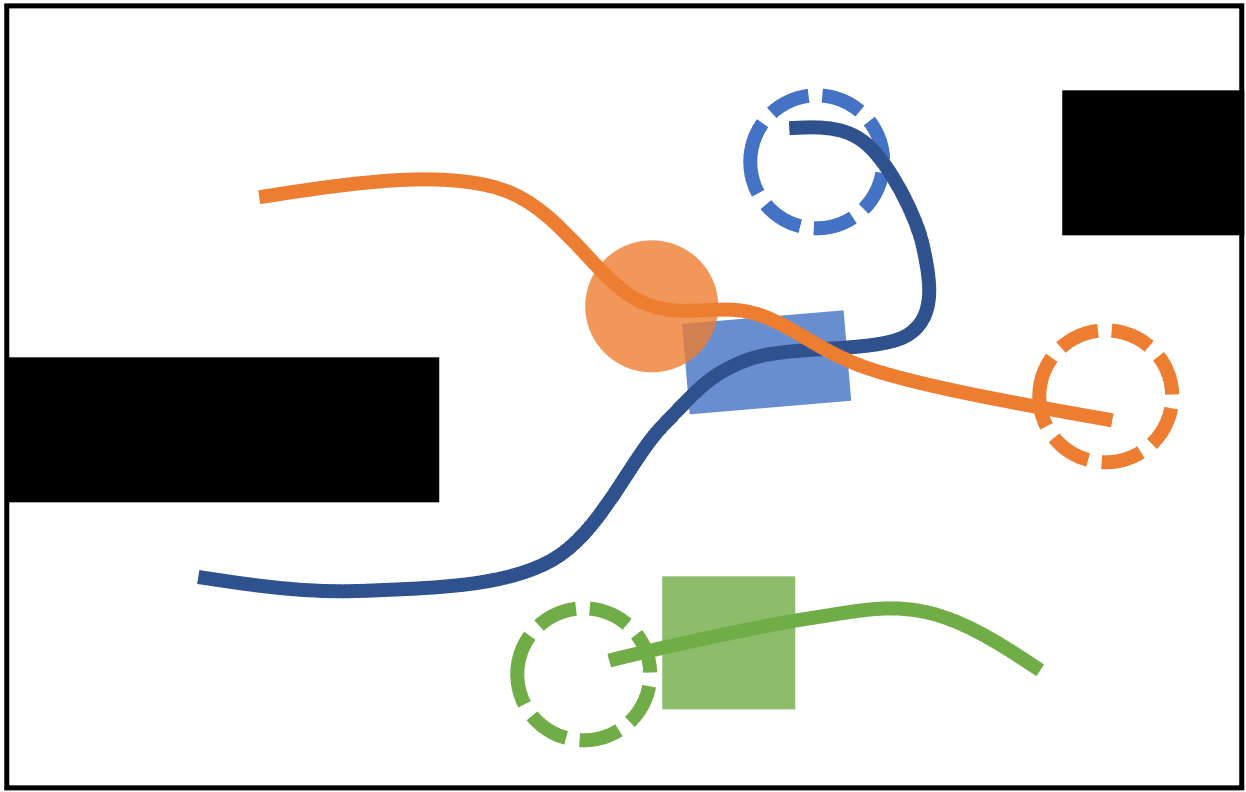}%
        \caption{Collision starts at $t_s$}%
        \label{fig:demo_conflict_ts}%
    \end{subfigure}\hfill
    \begin{subfigure}{0.24\textwidth}
        \centering
        \includegraphics[width=\textwidth]{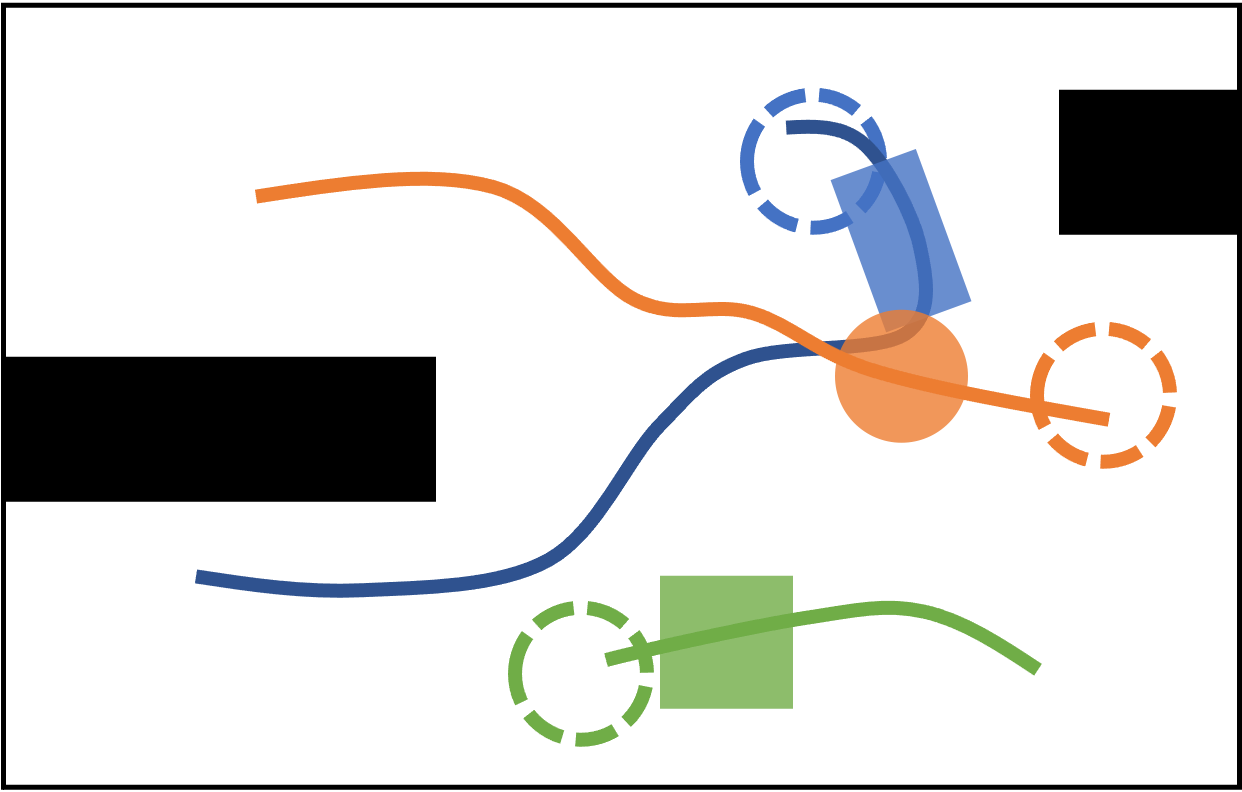}%
        \caption{Collision ends at $t_e$}%
        \label{fig:democ_conflict_te}%
    \end{subfigure}\hfill
    \begin{subfigure}{0.24\textwidth}
        \centering
        \includegraphics[width=\textwidth]{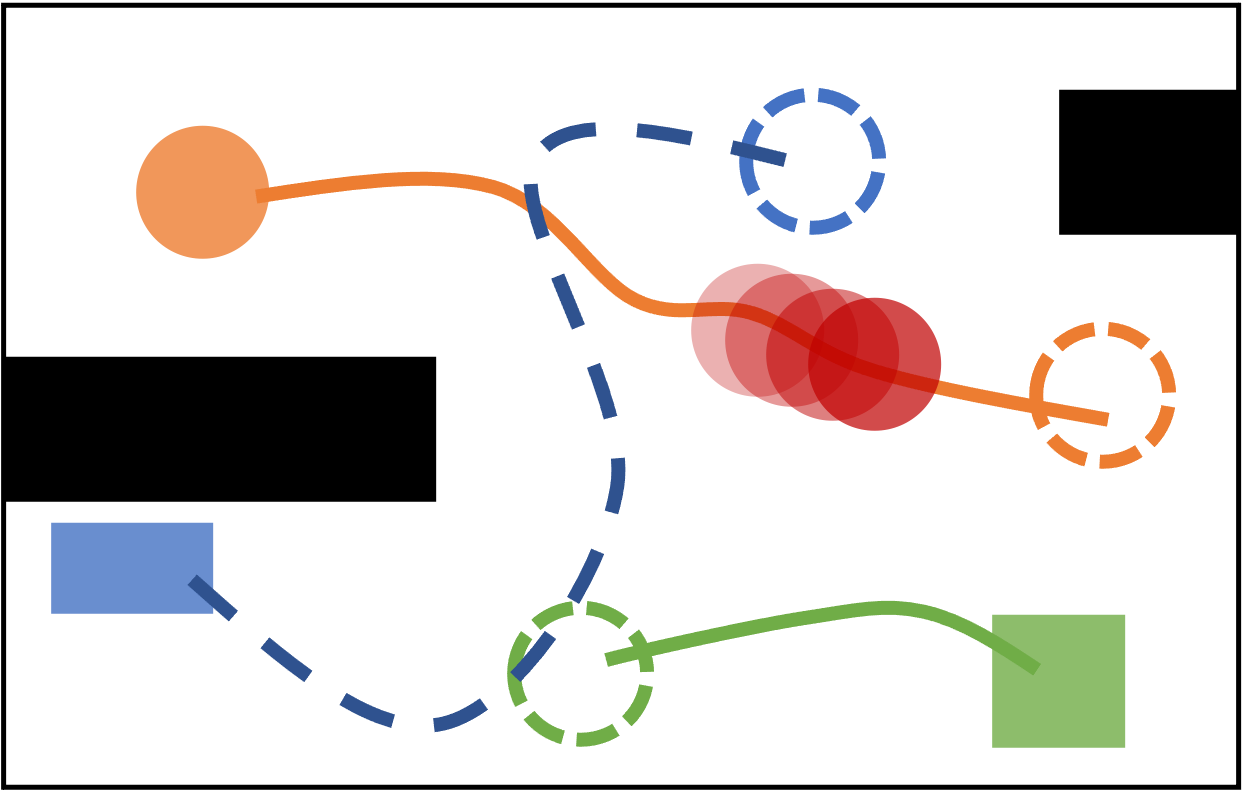}%
        \caption{Constraint on blue agent}%
        \label{fig:democ_conflict_constraint}%
    \end{subfigure}\hfill
    \caption{A visual representation of a conflict, collision, and constraint. Multiple conflicts occur during a single collision between the orange and blue robots within time range $t\in[t_s, t_e]$. The set of conflicts results in a single constraint for each robot, which are reasoned about separately. Constraint on orange agent not shown.}
    \label{fig:constraint}
\end{figure*}

\subsection{Conflict-Based Search for MAPF}
CBS is a two-level search on the space of possible MAPF plans, 
and consists of a high-level conflict-tree search and a low-level graph search. 
At the high-level, CBS keeps track of a \emph{constraint-tree}, in which each node represents a suggested plan, which might have collisions, referred to as \emph{conflicts}. Initially, a root node is obtained by using a low-level graph search algorithm, typically $A^*$, to find a path for each robot from start to goal, ignoring the other robots (hence the decentralized nature of the method).

At each iteration, CBS picks an unexplored node from the constraint-tree, based on some heuristic, and identifies the conflicts (namely collisions) in the node's corresponding plan.
Then, CBS attempts to resolve the conflicts by creating child nodes based on the conflicts, as follows: if robots $i$ and $j$ collide at time $t$ in vertex $v$ (deonted $\langle i, j, v, t \rangle$), then two children are created for the node, one with the constraint that robot $i$ cannot be in vertex $v$ at time $t$ (denoted $\langle i, v, t \rangle$), and the other dually for robot $j$. Then, in each child node, a low-level search is used to replan a path for the newly-constrained robot, given the set of constraints obtained thus far along the branch of the constraint tree.
This process repeats until either a non-colliding plan is found, or no new nodes are created in the constraint tree, 
at which point CBS returns that there is no solution.

\subsection{MRMP Constraints}
\label{subsec:MRMP Constraints}
CBS definitions of conflicts and constraints fail for MRMP because (i) robots do not occupy vertices, and (ii) robots could have unique bodies. Additionally, due to the continuous time nature of MRMP, collisions occur over time intervals rather than time points as in MAPF. Thus, formal definitions of collisions and constraints for MRMP are needed.

Given two trajectories $T_i$ and $T_j$ for robots $i$ and $j$, respectively, recall that a \emph{collision} occurs between robots $i$ and $j$ at every time point that Condition~\ref{cond:robotColliision} is violated.  The duration of a collision is the \emph{continuous} time interval the two robots are in collision as they 
evolve along their respective trajectories,
and we define a \emph{conflict} to be $K=\langle i,j,[t_s,t_e]\rangle$, where $t_s$ and $t_e$ are the start and end time of this interval.
Note that there may be multiple conflicts between trajectories $T_i$ and $T_j$, in which case their time intervals are disjoint.

A conflict $K=\langle i,j,[t_s,t_e]\rangle$ is resolved by imposing a \emph{constraint} $\mathcal{C}$ on either robot $i$ or robot $j$. Intuitively, a constraint on robot $i$ 
requires 
it to avoid colliding with robot $j$ for all $t_r\in[t_s, t_e]$ by viewing $\mathcal{S}_j(T_j(t_r))$ as a moving (time-dependent) obstacle. 
Formally, the constraint on robot $i$ for conflict $K$ is $\mathcal{C}_i=\{\langle i, \mathcal{S}_j(T_j(t_r)), t_r \rangle\ :\ t_r\in [t_s,t_e]\}$
Thus, $\mathcal{C}_i$ treats $\mathcal{S}_j(T_j(t_r))$ as a moving obstacle rather than a static one for the entire time duration of the conflict (see Fig.~\ref{fig:constraint}).  We dually define the constraint $\mathcal{C}_j$ for robot $j$.
This splitting is the crux of the performance of the approach.

Next, we describe how to resolve a conflict by re-planning for an individual robot with the corresponding constraint.

\subsection{Single-Robot Planning under MRMP Constraints}

To (re-)plan for individual robots, we can use any existing sampling-based, tree-search algorithm (e.g. RRT) since they are efficient in finding kinodynamically feasible trajectories.
A generic form of such planner $\mathcal{X}$ begins by initializing a graph $G$ at the initial state. Then, for $N$ iterations, the algorithm samples a state $\mathbf{x}_{rand}$, extends the motion tree towards $\mathbf{x}_{rand}$ by sampling an input and integrating \eqref{eqn:dyn}, and adds the trajectory to the graph after verifying that it is collision free. The process continues until a trajectory from the initial state to the goal region is found, or $N$ iterations is reached. 


Such planners perform well for single-query, single-robot kinodynamic planning, but off-the-shelf, they are unable to guarantee the satisfaction of a set of constraints as defined in 
 Sec.~\ref{subsec:MRMP Constraints}.
To this end, we present \Alg{Constrained-$\mathcal{X}$} (\Alg{CSTR-$\mathcal{X}$}, shown in Alg.~\ref{alg:cstrX}), an extension of planner $\mathcal{X}$ 
for finding trajectories that do not violate a set of constraints $\mathcal{C}$.

\begin{algorithm}[b]
\KwIn{($X$, $U$, $f$, $X^G$, $\mathcal{O}$, $\mathbf{x}_0$, N, $\mathcal{C}$, $\mathcal{T}$)}
\KwOut{Collision free trajectory $T$ that respects $\mathcal{C}$}
\eIf{$\mathcal{T}$ is empty}
{
    $G=\{V\leftarrow x_0, E\leftarrow \emptyset\}$\label{ln:newG}\;
}
{
    $G=\mathcal{T}$\label{ln:growT}\;
}
 \For{N iterations}
 {
    $\mathbf{x}_{rand}, \mathbf{x}_{near}\leftarrow$ \Alg{sample}($X$, $G.V$)\label{ln:sampleCSTR}\;
    $\mathbf{x}_{new} \leftarrow$ \Alg{extend}($\mathbf{x}_{rand}$, $\mathbf{x}_{near}$, $U$, $f$)\label{ln:extendCSTR}\;
    \If{isValid($\overrightarrow{\mathbf{x}_{near},\mathbf{x}_{new}}$, $\mathcal{O}$, $\mathcal{C}$)\label{ln:collCheckCSTR}}
    {
        $V\leftarrow V\cup \{\mathbf{x}_{new}\}$;
        $E\leftarrow E \cup \{\overrightarrow{\mathbf{x}_{near},\mathbf{x}_{new}}\}$\label{ln:addCSTR}\;
        \If{$\mathbf{x}_{new} \in X^G$\label{ln:golCheckCSTR}}
        {
            \KwRet{$\overrightarrow{\mathbf{x}_{0},\mathbf{x}_{new}}$\label{ln:retSolCSTR}}            
        }
    }
 }
\KwRet{$G(V,E)$ and ${\mathcal C}_{\max}$}
\caption{Constrained-$\mathcal{X}$ (CSTR-$\mathcal{X}$)}
\label{alg:cstrX}
\end{algorithm}

\Alg{CSTR-$\mathcal{X}$} 
takes two additional inputs than $\mathcal{X}$: a set of constraints $\mathcal{C}$, and an existing graph (motion tree) $\mathcal{T}$.
If $\mathcal{T}$ is non-empty, \Alg{CSTR-$\mathcal{X}$} utilizes the existing graph (Line~\ref{ln:growT}). Otherwise, it initializes a new graph $G$ with a single initial state $\mathbf{x}_0$ (Line~\ref{ln:newG}). Then, to check if a trajectory is valid, \Alg{CSTR-$\mathcal{X}$} checks against both static obstacles $\mathcal{O}$ and constraints (moving obstacles) in $\mathcal{C}$ (Line~\ref{ln:collCheckCSTR}). 
Similar to existing validity checking functions, checking against $\mathcal{C}$ can be done efficiently, e.g., by discretizing the continuous trajectory $\overrightarrow{\mathbf{x}_{near},\mathbf{x}_{new}}$ into a set of points $\{x(\tau_1), \ldots, x(\tau_L)\}$ by sampling time at a fine resolution $\Delta t$, and checking $\mathcal{S}(\mathbf{x}(\tau_l))$ against all $\mathcal{C}^r\in \mathcal{C}$ for all $1\leq l \leq L$ such that $\tau_l$ is the time range of $\mathcal{C}^r$, i.e., $\tau_l \in [t^r_s, t^r_e]$.

The output of \Alg{CSTR-$\mathcal{X}$} is a trajectory that respects all the constraints, if one is found in $N$ iterations; otherwise, it returns the full motion tree, as well as the constraint ${\mathcal C}_{\max}$ that was violated the most.  Our \Alg{K-CBS} algorithm uses the motion tree to continue planning when needed, and uses ${\mathcal C}_{\max}$ to identify tangled plans, as discussed below.
Note that \Alg{CSTR-$\mathcal{X}$} defaults to Planner \Alg{$\mathcal{X}$} if both additional parameters $\mathcal{T}$ and $\mathcal{C}$ are empty sets.  



\subsection{Kinodynamic Conflict-Based Search (\Alg{K-CBS})}
\begin{algorithm}[!ht]
\KwIn{$\{\mathcal{M}_i\}_{i=1}^{k}$, N, B}
\KwOut{Collision-free trajectories $\{T_1, \ldots, T_k\}$}
$Q, n_0, p_0, \leftarrow \emptyset$\;\label{ln:initKCBS}
\For{every robot $i$\label{ln:foreveryrobot}}
{
    $p_0 \leftarrow p_0 \cup$ \Alg{CSTR-X}($\mathcal{M}_i$, $\mathcal{O}$, $\mathbf{x}_{i,0}$, $\infty$, $\emptyset$, $\emptyset$)\label{ln:planInit}
    \;
}
$n_0$.plan $\leftarrow$ $p_0$; $Q$.\Alg{add}($n_0$) \label{ln:rootNode}\;
\While{solution not found \label{ln:mainBegin}}
{
    $c \leftarrow$ $Q$.\Alg{top}()\label{ln:selectKCBS}\;
    \eIf{$c$.plan is empty \label{ln:retry}}
    {
        $Q$.\Alg{pop}(); $\mathcal{C}^i \leftarrow$ $c.\mathcal{C}$; $\mathcal{T}^i \leftarrow$ $c.\mathcal{T}$\label{ln:retrySetup}\;
        $T_i',\mathcal{C}_{\max} \leftarrow$ \Alg{CSTR-X}($\mathcal{M}_i$, $\mathcal{O}$, $\mathbf{x}_{i,0}$, N, $\mathcal{T}^i$, $\mathcal{C}^i$)\label{ln:endRetry}\;
        \eIf{$T_i'$ is a trajectory}
        {
            $(c$.plan$\setminus T_i) \cup T_i'$\label{ln:retrySuc}\;
        }
        {
            $c$.tree $\leftarrow T_i'$ \label{ln:retryFail}\; 
            \If{\Alg{shouldMerge}($\mathcal{C}_{\max}$, B)\label{ln:mergeIf1}}
             {
                $\{\mathcal{M}_i\}_{i=1}^{k-1}\leftarrow$ \Alg{merge}($\{\mathcal{M}_i\}_{i=1}^{k}$,$\mathcal{C}_{\max}$)\label{ln:merge1}\;
                 $P\leftarrow$\Alg{K-CBS}($\{\mathcal{M}_i\}_{i=1}^{k-1}$, N, $B$)\;
                 \KwRet{$P$ \label{ln:mergeIf2}}
             }
        }
        $Q$.\Alg{add}($c$)\label{ln:retryAdd}\;
    }
    {
        K $\leftarrow$ \Alg{validatePlan}($c$.plan)\label{ln:validate}\;
        \uIf{K is empty\label{ln:isKempty}}
        {
            \KwRet{$c$.plan\label{ln:retPlan}}
        }
        \uElseIf{\Alg{shouldMerge}(K, B)\label{ln:shouldMerge}}
        {
            $\{\mathcal{M}_i\}_{i=1}^{k-1}\leftarrow$ \Alg{merge}($\{\mathcal{M}_i\}_{i=1}^{k}$, $K$)\label{ln:merge}\;
            $P\leftarrow$\Alg{K-CBS}($\{\mathcal{M}_i\}_{i=1}^{k-1}$, N, $B$)\label{ln:solveMeta}\label{ln:planMeta}\;
            \KwRet{$P$}
        }
        \Else{
            $Q$.\Alg{pop}()\label{ln:branchPop}\;
            \For{every robot $i$ in K\label{ln:branchForLoop}}
            {
                $\mathcal{C}^i \leftarrow$ $c.\mathcal{C} \; \cup$ K.\Alg{getCSTR}($i$); $c_{new} \leftarrow \emptyset$\label{ln:branchSet-up}\;
                $T_i' \leftarrow$ \Alg{CSTR-X}($\mathcal{M}_i$, $\mathcal{O}$, $\mathbf{x}_{i,0}$, N, $\emptyset$, $\mathcal{C}^i$)\label{ln:branchReplan}\;
                \eIf{$T_i'$ is a trajectory}
                {
                    $c_{new}\leftarrow(c$.plan$\setminus T_i) \cup T_i'$\label{ln:banchSuc}\; 
                }
                {
                    $c_{new}$.tree $\leftarrow T_i'$\label{ln:branchFail}\; 
                }
            $Q$.\Alg{add}($c_{new}$)\label{ln:branchAdd}\;   
            }
        }
    }
    \label{ln:mainEnd}
}
\caption{Kinodynamic-CBS (K-CBS)}
\label{alg:K-CBS}
\end{algorithm}
We now introduce the high-level algorithm of \Alg{K-CBS} (Alg.~\ref{alg:K-CBS}). 
Recall that the input to Problem~\ref{prob} comprises a state space $X_i$, input space $U_i$, vector field $f_i$, body $\mathcal{S}_i$, initial state $\mathbf{x}_{i,0}$, and goal region $X_i^G$ for every robot $i\in\{1, \ldots, k\}$. For brevity, we define a model $\mathcal{M}_i=\{X_i, U_i, f_i, \mathbf{x}_{i,0}, X_i^G\}$ as the inputs pertaining to robot $i$. The input to \Alg{K-CBS} is then a model for every robot $\{\mathcal{M}_i\}_{i=1}^{k}$, 
a parameter $N$ of the maximum number of iterations of \Alg{CSTR-$\mathcal{X}$}, and a merging parameter $B$. The output of \Alg{K-CBS} is a set of trajectories $\{T_1, \ldots, T_k\}$ 
that is the solution to the MRMP problem.

\Alg{K-CBS} begins by calculating trajectories for every robot individually by calling \Alg{CSTR-$\mathcal{X}$} with an empty initial tree and empty set of constraints. These trajectories serve as the initial plan for the root node of the constraint-tree, which is added to a priority-queue $Q$ (Lines~\ref{ln:initKCBS}-~\ref{ln:rootNode}). 
The cost of node $n\in Q$ is the sum of control durations over all trajectories in the plan of $n$ (and $\infty$ if $n$ does not contain a plan yet). The smaller the cost of $n$, the higher its priority in $Q$.

At every iteration of \Alg{K-CBS}, a node $c \in Q$ with minimal cost is selected (Line~\ref{ln:selectKCBS}). If $c$ has a plan, it is evaluated for conflicts (Line~\ref{ln:validate}). If no conflict exists, then the plan is returned as the solution. 
Otherwise, an arbitrary conflict $K$ is selected and for each robot in $K$, and a constraint is generated.
Then, \Alg{CSTR-$\mathcal{X}$} is called to calculate a valid trajectory, storing the result of the search in a new node, which is added to $Q$ (Lines~\ref{ln:branchPop}-~\ref{ln:branchAdd}). If $c$ has no plan, another call to \Alg{CSTR-$\mathcal{X}$} is performed in an attempt to obtain such a plan (Lines~\ref{ln:retry}-~\ref{ln:retryAdd}). 

As \Alg{K-CBS} explores the space of all plans, it is possible that many conflicts are found between a particular pair of robots.  This can indicate that their motions are ``coupled,'' and hence, it is difficult to plan for them separately. In that case, \Alg{K-CBS} can merge the pair into a meta-robot and plan for an MRMP problem with $k-1$ robots 
(Lines~\ref{ln:mergeIf1}--\ref{ln:mergeIf2} and~\ref{ln:shouldMerge}--\ref{ln:planMeta}). We elaborate on this process below.

\begin{figure}[!ht]
    \centering
    \includegraphics[scale=0.28]{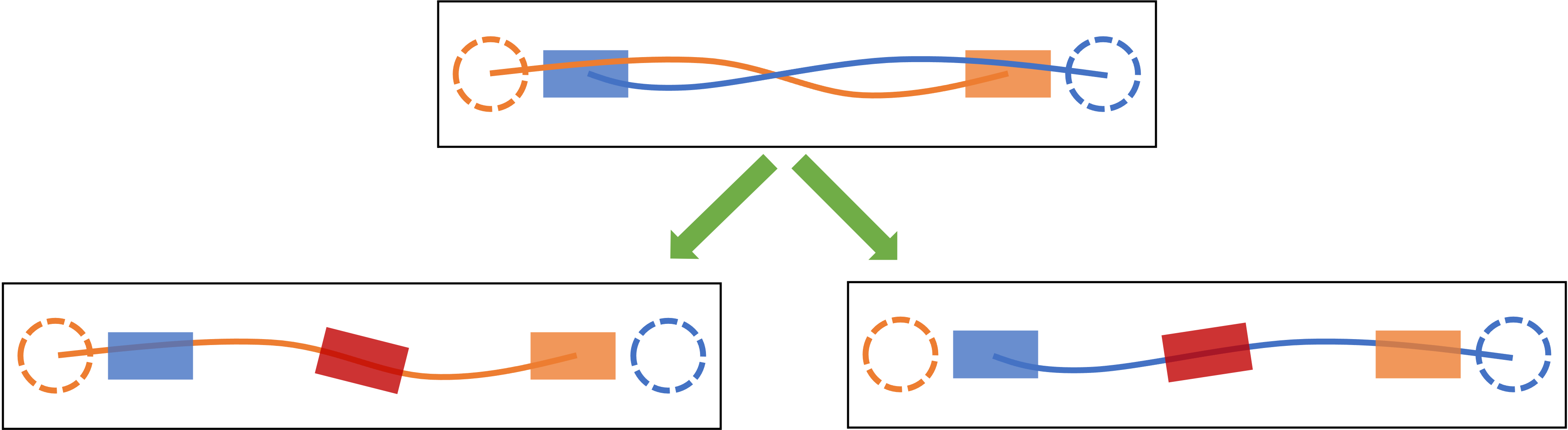}
    \caption{K-CBS tree. Robots must cross a narrow corridor.}
    \label{fig:incomplete}
\end{figure}

\begin{figure*}
    \centering
    \begin{subfigure}{0.25\textwidth}
        \centering
        \includegraphics[scale=0.26]{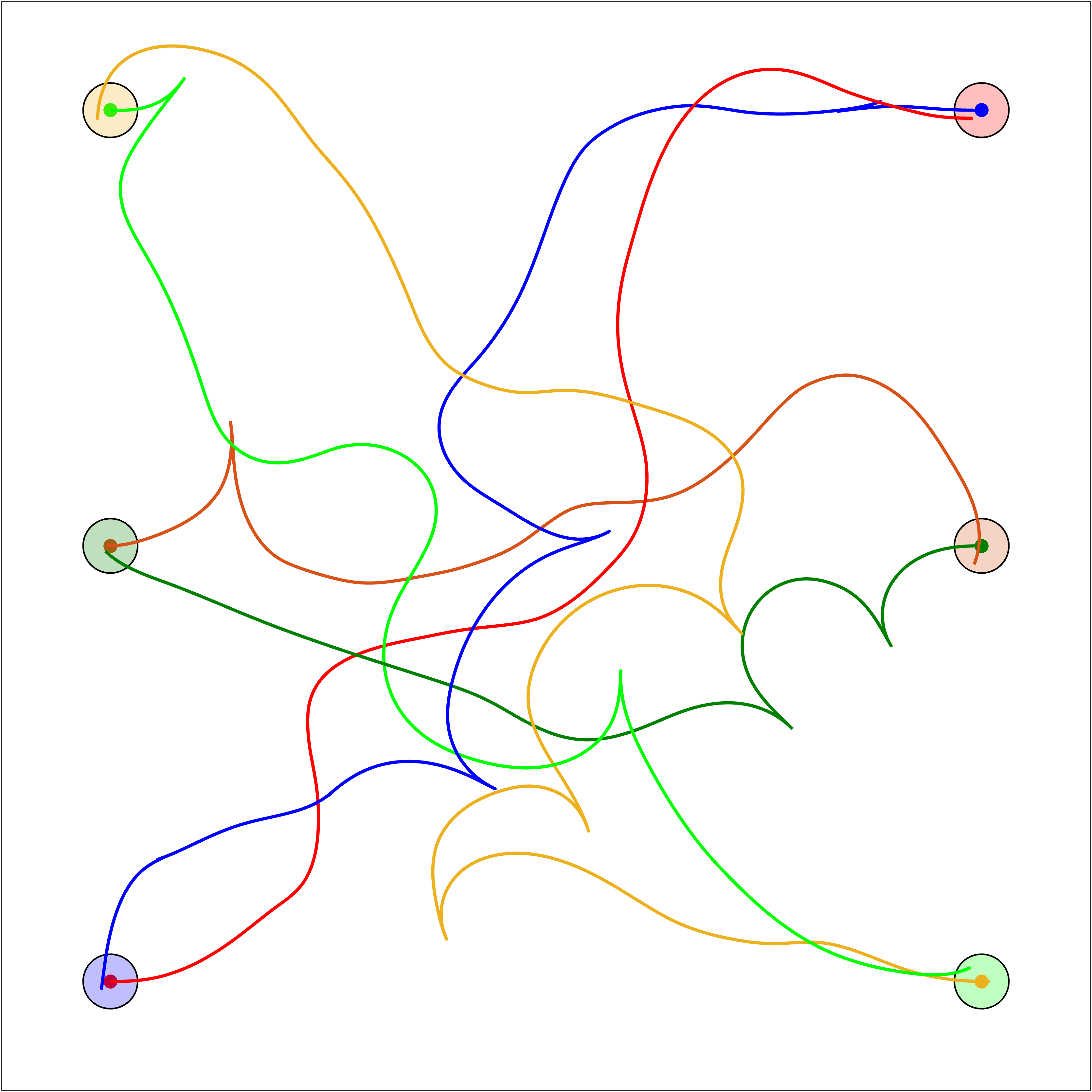}%
        \caption{Open Env. ($10\times 10$)}%
        \label{fig:envOpen}%
    \end{subfigure}\hfill
    \begin{subfigure}{0.25\textwidth}
        \centering
        \includegraphics[scale=0.26]{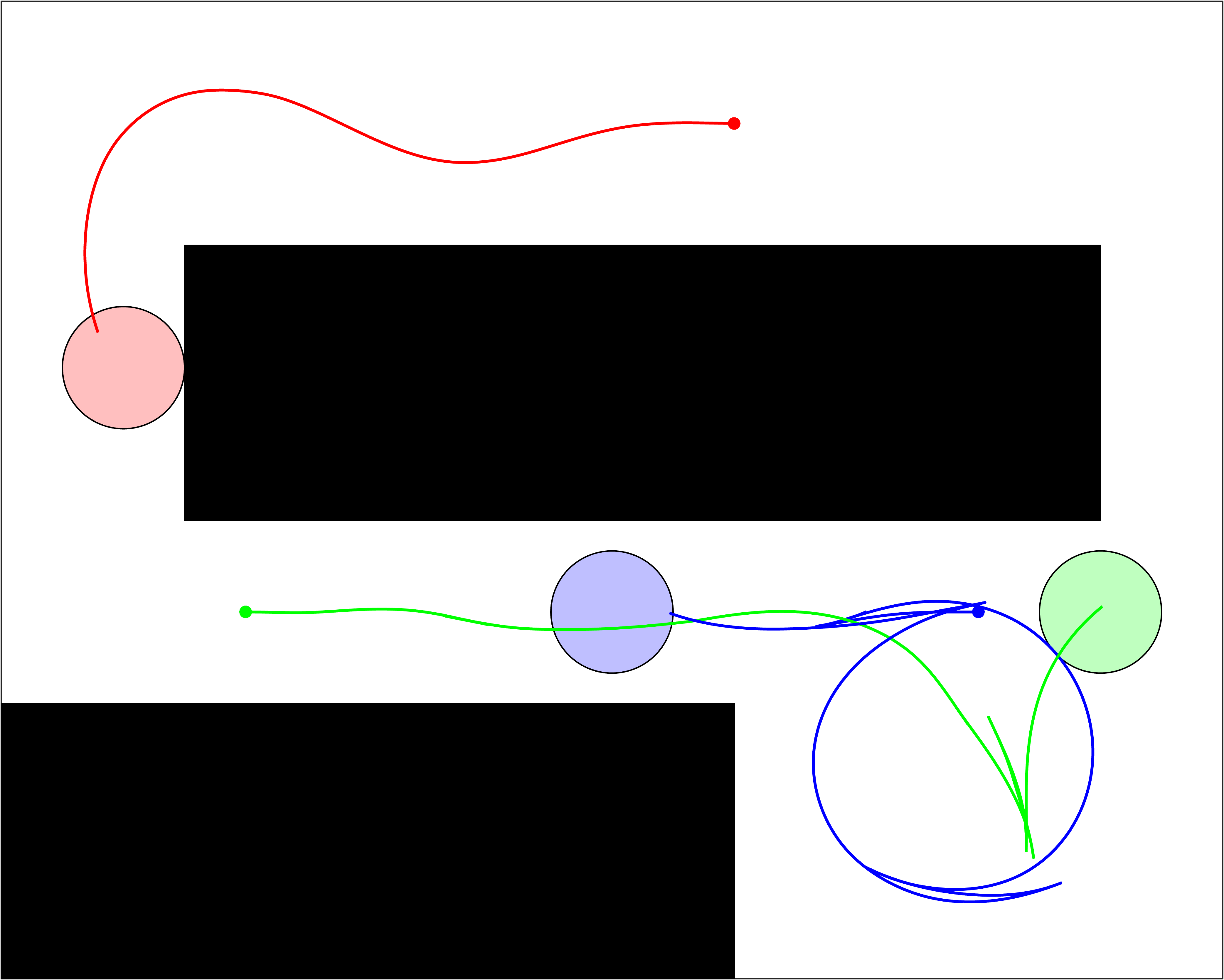}%
        \caption{Narrow Env. ($5\times 4$)}%
        \label{fig:envNarrow}%
    \end{subfigure}\hfill
    \begin{subfigure}{0.25\textwidth}
        \centering
        \includegraphics[scale=0.26]{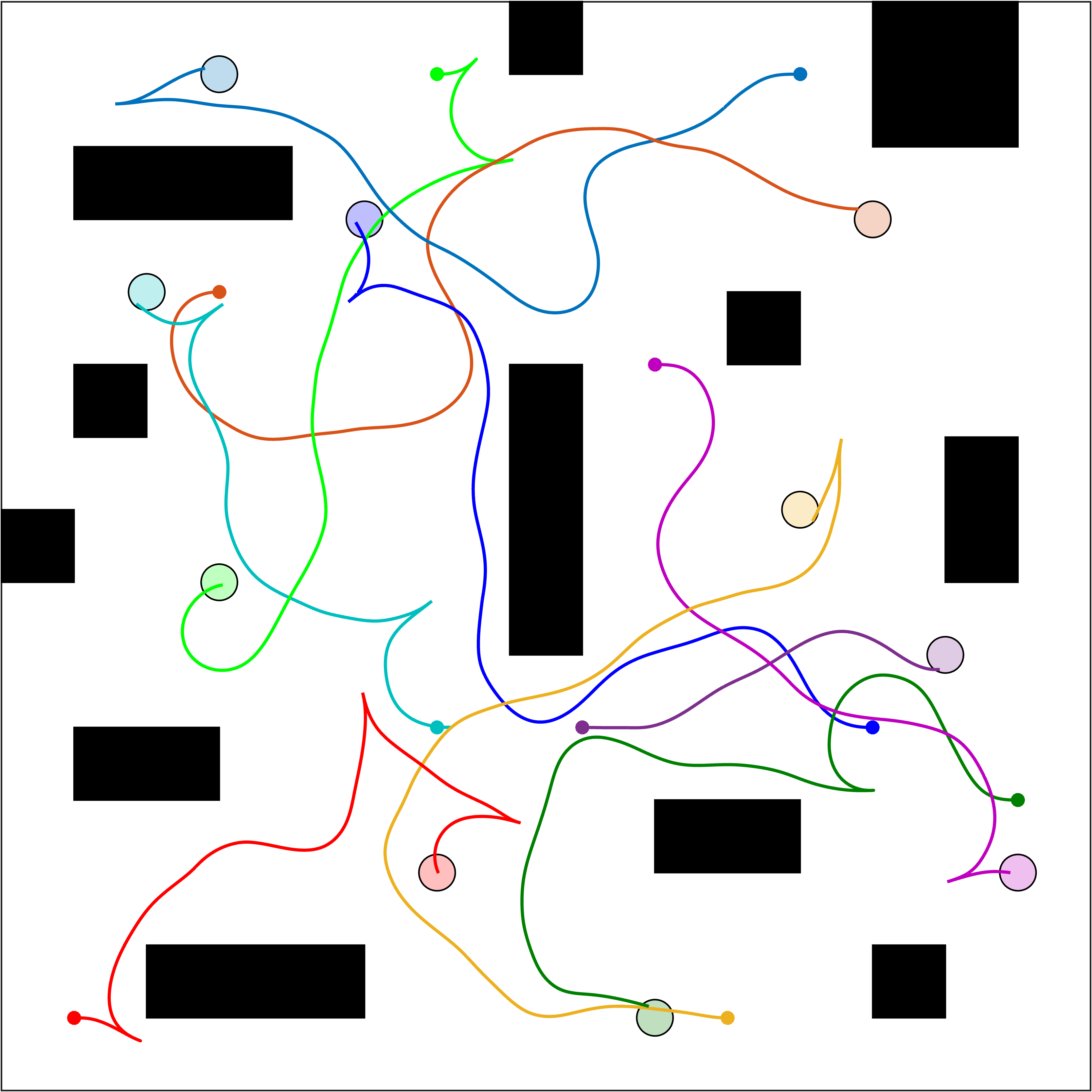}%
        \caption{Cluttered Env. ($15\times 15$)}%
        \label{fig:envCluttered}%
    \end{subfigure}\hfill
    \begin{subfigure}{0.25\textwidth}
        \centering
        \includegraphics[scale=0.26]{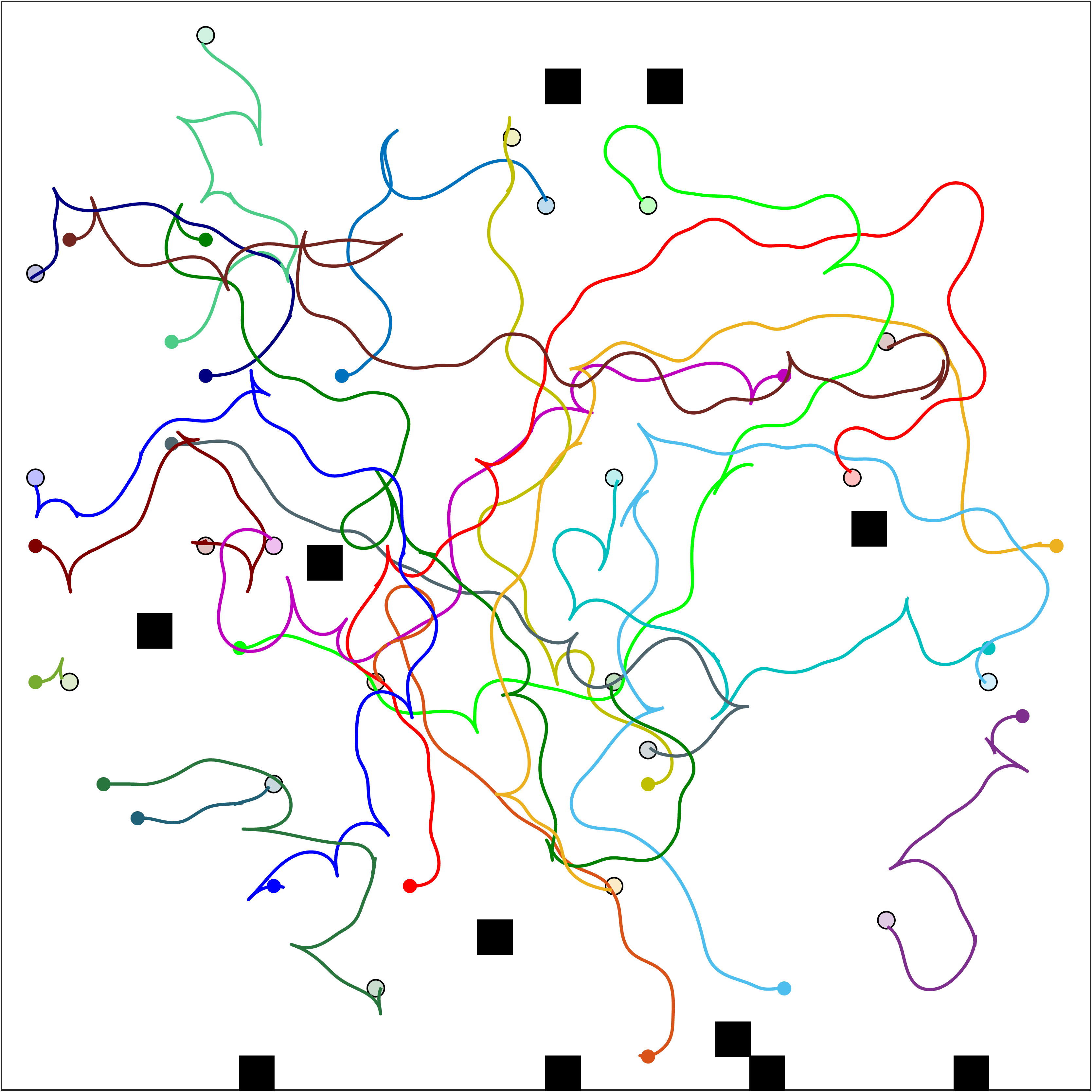}%
        \caption{Large Env. ($33\times 33$)}%
        \label{fig:envLarge}%
    \end{subfigure}\hfill
    \caption{MRMP benchmark environments. The initial states are shown with small circles and goal regions with large circles.}
    \label{fig:envs}
\end{figure*}

\subsubsection*{Merge and Restart}
\label{subsec:MR}
To clarify the concept of ``coupled'' robots,
consider the example 
in Fig.~\ref{fig:incomplete}, where two robots must cross a narrow corridor that is just wide enough for the both of them. Assume that both robots have constant velocity toward their respective target and they only control their steering angle. When presenting \Alg{K-CBS} with this problem instance, the root node of the constraint tree could look like the one seen in at the top of Fig.~\ref{fig:incomplete}. Naturally, since both robots attempt to traverse through the middle of the corridor, there is a collision. This creates a branching of the root node into two child nodes. One with a constraint on the blue robot to avoid the orange robot, and dually for the orange robot (constraints are visualized in red). 
In both cases, the constraint entirely prevents a valid plan for either agent, although  a-priori a plan does exist.
If this behavior is left unresolved, \Alg{K-CBS} would never find a valid plan. Thus, we employ the \emph{merge and restart} technique, borrowed from the discrete setting~\cite{boyarski2015icbs}. Note that in the discrete setting, this is merely a speedup heuristic, whereas for the continuous case it is crucial for completeness.
While many merge techniques are available, we utilize a simple, experimentally effective merge policy from~\cite{sharon2012meta}, whereby robots $i$ and $j$ are merged if the number of conflicts between them exceeds some pre-defined threshold $B$. 

The merge policy for \Alg{K-CBS} is shown in Alg.~\ref{alg:K-CBS} (Lines~\ref{ln:mergeIf1}--\ref{ln:mergeIf2} and~\ref{ln:shouldMerge}--\ref{ln:planMeta}). If the \Alg{shouldMerge}() procedure returns true via the requirement above, the two robots are composed together (e.g., Line~\ref{ln:merge}) and another instance of \Alg{K-CBS} is solved for $k-1$ robots where one of the robots in the newly formed problem is a meta-robot (e.g., Line~\ref{ln:solveMeta}). Note that in the case of Fig.~\ref{fig:incomplete}, eventually \Alg{K-CBS} merges the two robots together and plans trajectories that allow both robots to navigate the corridor simultaneously. Also note that we count as conflicts both nodes where a conflict occurs, as well as nodes where no plan is found, via $\mathcal{C}_{\max}$.

\subsection{Completeness}
\label{sec:completness}
Algorithm
\Alg{K-CBS} inherits the probabilistic completeness property of the underlying Planner $\mathcal{X}$.  
\begin{theorem}[Completeness]
    \label{thm:completeness}
    \Alg{K-CBS} is probabilistically complete if Planner $\mathcal{X}$ is probabilistically complete.
\end{theorem}
\begin{proof}
Ultimately, probabilistic completeness follows from the fact that in the worst case, \Alg{K-CBS} merges all the robots and becomes centralized, and using the completeness of Planner $\mathcal{X}$. 
To show this observe that at each iteration either (1) a merge happens, (2) a conflict node is extended, or (3) a node without a plan receives more planning time. (1) can only happen $k$ times along a branch, (2) can happen at most $k(k-1)B$ times before a merge is made, and (3) happens at most $B$ times for each node. Therefore, in each branch of the tree, eventually either a plan is found or all the agents merge and \Alg{K-CBS} becomes centralized.
\end{proof}


Finally, we remark that it may be possible to further improve the performance of \Alg{K-CBS} using heuristics.  Nonetheless, this requires careful treatment and deep understanding of the MRMP problems.  Not every successful heuristic used in MAPF is beneficial in the continuous domain.  For instance, work \cite{boyrasky2015don} shows that, in CBS, significant speedups can be achieved by \emph{bypassing} (BP) a split of a constraint-tree node by generating a single constraint and attempting to resolve it by calling the low-level planner in the same node. Our studies however do not show such improvements by incorporating BP in \Alg{K-CBS}.



\section{Experiments and Benchmarks}
\label{sec:Exp}

Here, we show the performance of \Alg{K-CBS} in four different environments 
(Open, Narrow, Cluttered, and Large) shown in Fig.~\ref{fig:envs}
with various numbers of robots, each with 2nd-order car dynamics (5-dimensional state space):
\begin{align*}
    \dot x= v \cos \theta, \; \dot y = v \sin \theta, \; \dot \theta = \frac{v}{l} \tan \phi, \;
    \dot v = u_1, \;  \dot \phi = u_2
\end{align*}
where $x$ and $y$ define the position, $\theta$ is the orientation, $v$ is the velocity, $\phi$ is the steering angle, $u_1$ is the acceleration rate and $u_2$ is the steering rate. 
Each robot has a rectangular shape with length $l=0.7$ and width $w=0.5$.

In each environment, we benchmarked the performance and scalability of \Alg{K-CBS} against two baseline approaches: centralized \Alg{RRT} (\Alg{cRRT}) and prioritized \Alg{RRT} (\Alg{pRRT})~\cite{lavalle2006planning}.\footnote{Comparison against  \cite{Le:ICAPS:2017,Le:RAL:2019} (without the PID controller) did not materialize due to unavailability of the code and complicacy
of the algorithm.}  
\Alg{cRRT} is a centralized method that plans in the composite space of all the robots, whereas \Alg{pRRT} 
is a decentralized algorithm that plans for each robot $i$ in turn, treating the plans for robots $1, \ldots, i-1$ as timed obstacles.
We also varied the merging parameter $B$ to show its effect on the performance of \Alg{K-CBS}.
Every benchmark consisted of $100$ problem instances with a timeout of $5$ minutes. 
The comparison metrics are \emph{success rate}, \emph{computation time} (of successful runs), and \emph{merge rate}. 
The results are shown in Fig.~\ref{fig:bench_open}-\ref{fig:bench_cluttered} and Table~\ref{tab:benchmarkTable}.  

Our implementation of \Alg{K-CBS} utilized the classical RRT as the low-lever planner. The code is publicly available on GitHub~\cite{sourceCode}.
The benchmarks were performed on 
AMD Ryzen 7 3.9 GHz CPU and 64 GB of RAM.

\begin{figure*}%
    \centering
        \includegraphics[width=.8\textwidth]{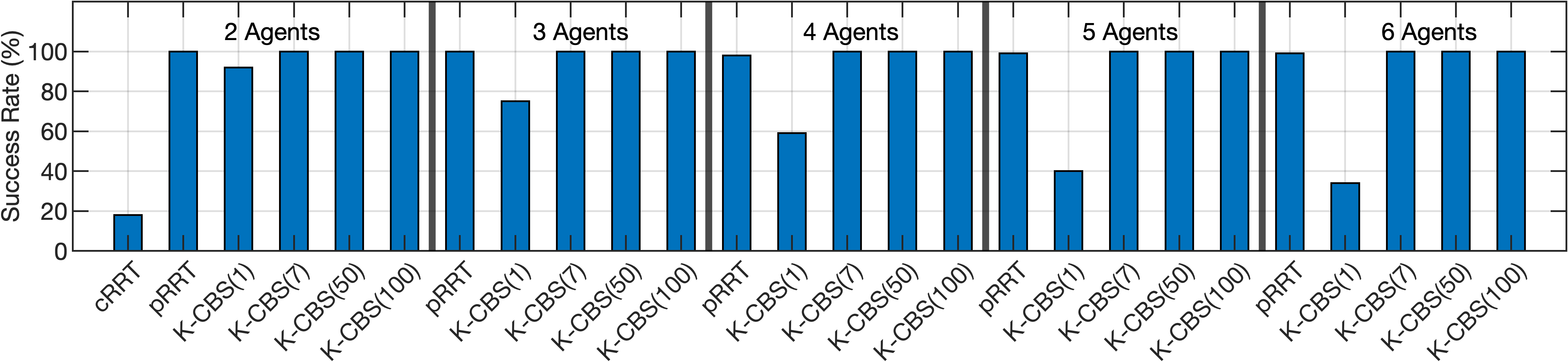}%
    
        \includegraphics[width=.8\textwidth]{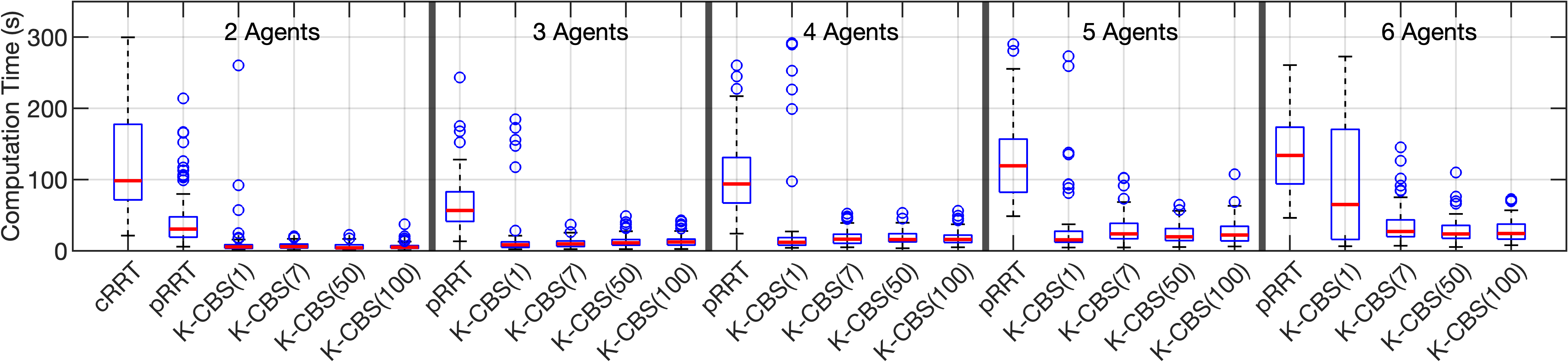}%
    \caption{Benchmark results for the Open environment in Fig.~\ref{fig:envOpen} 
    }
    \label{fig:bench_open}
\end{figure*}
\begin{figure*}%
\centering
    \includegraphics[width=.32\textwidth]{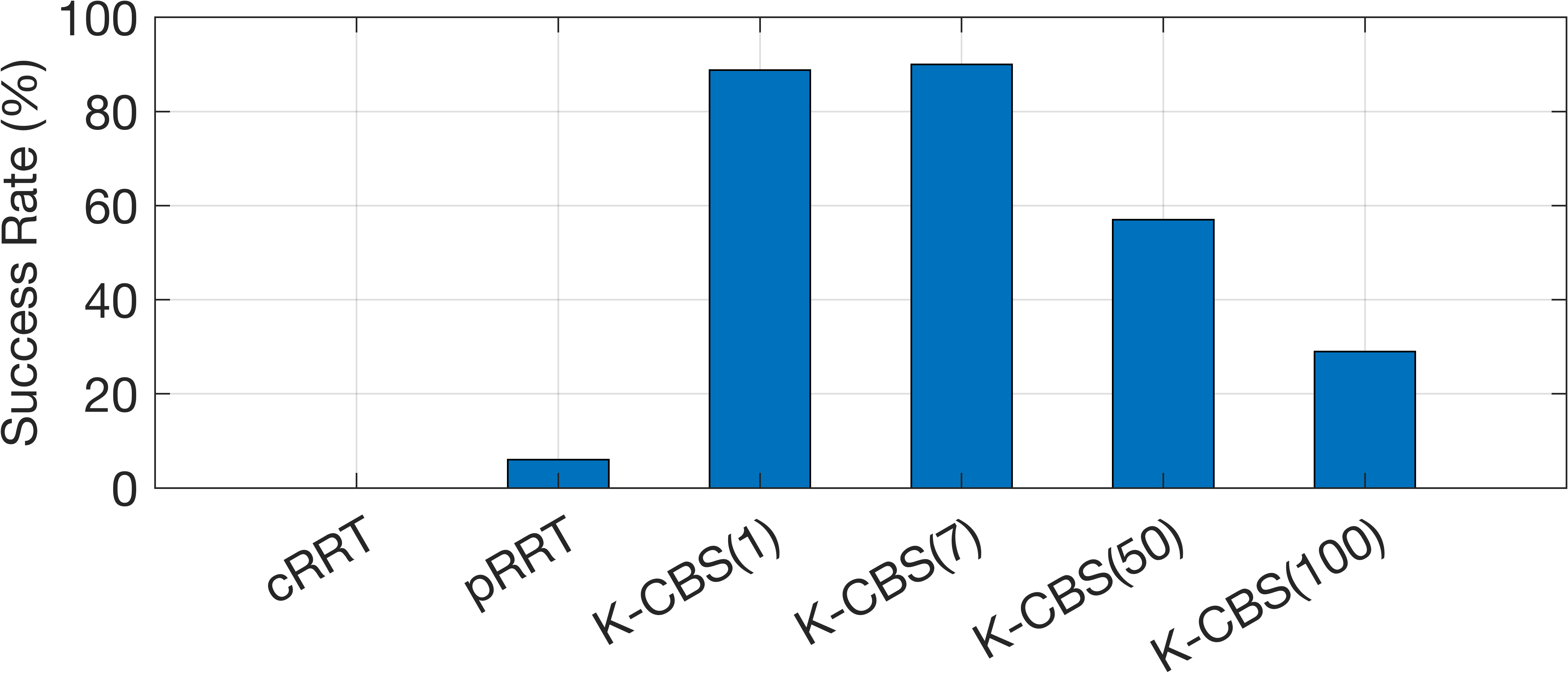}%
    \hfill
    \includegraphics[width=.32\textwidth]{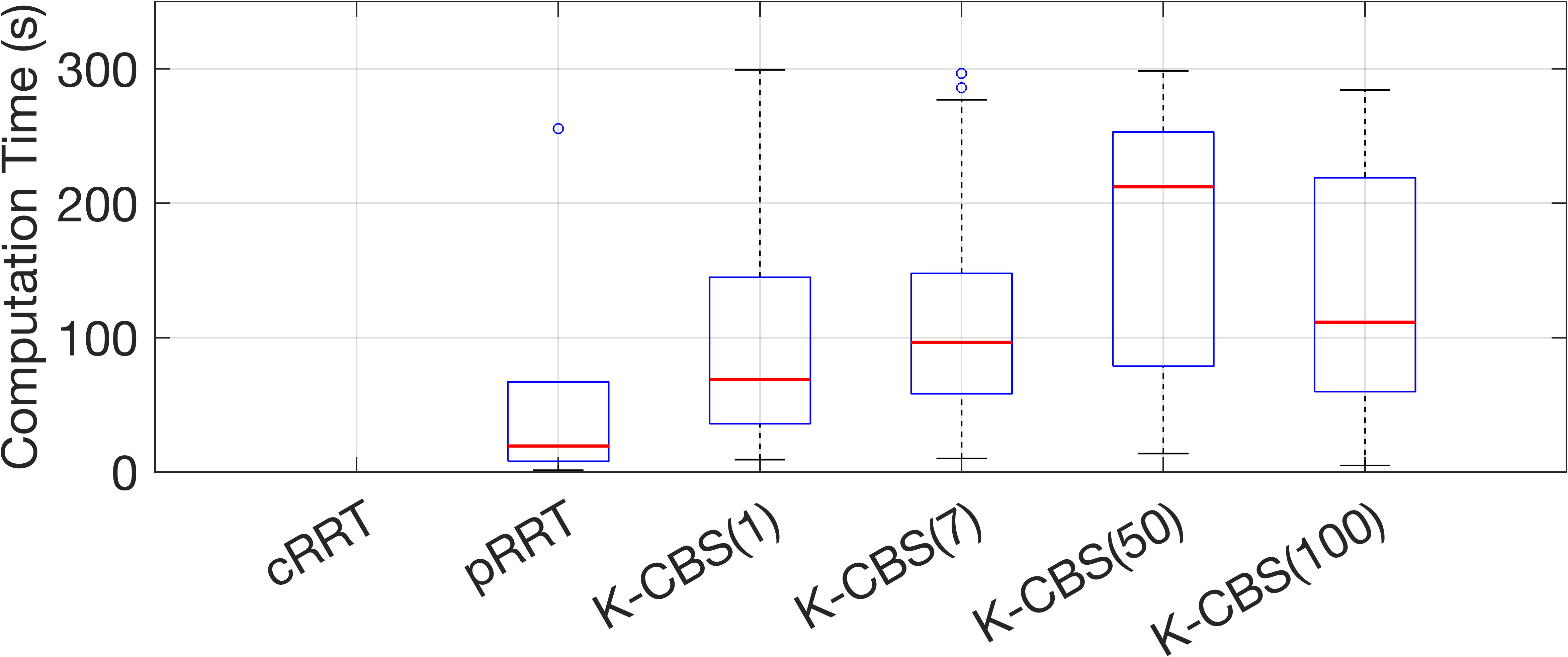}%
    \hfill
    \includegraphics[width=.32\textwidth]{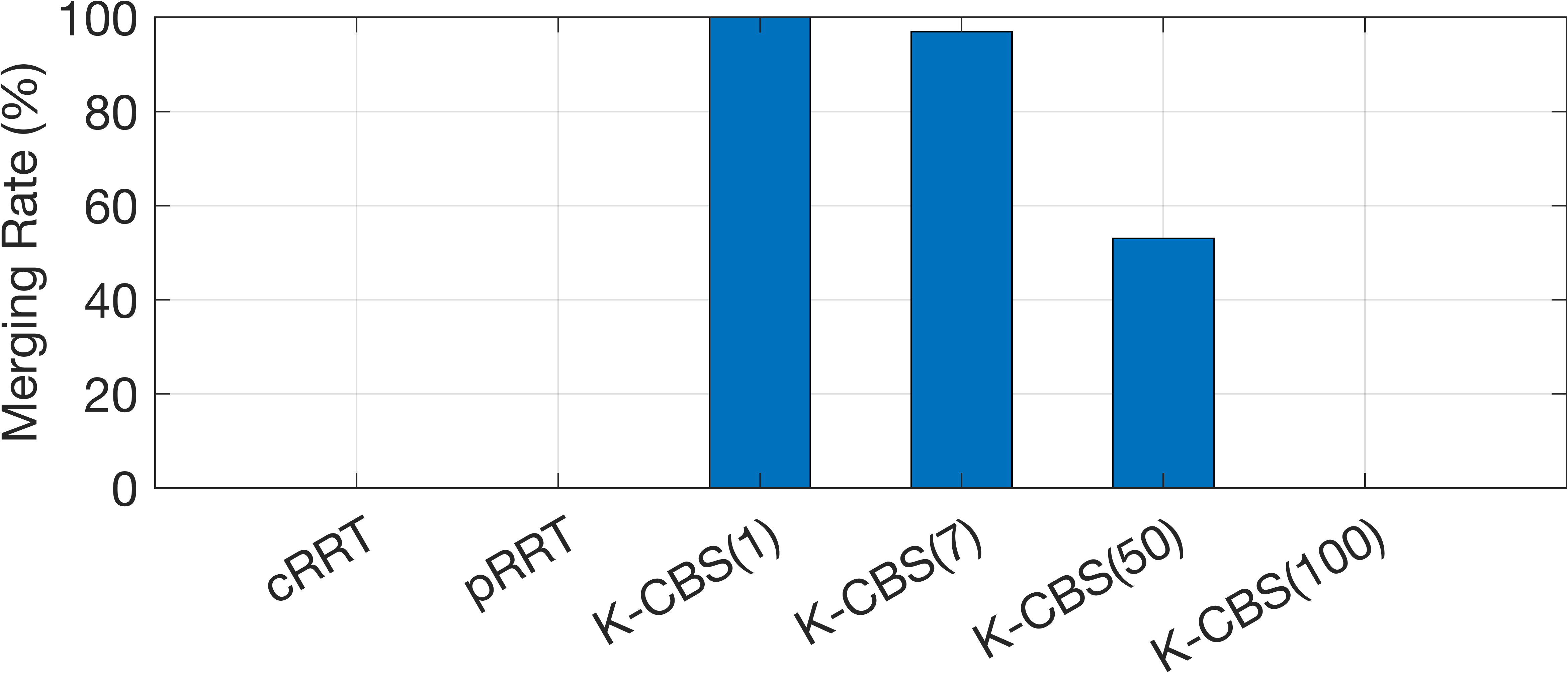}%
    \hfill
    \caption{Benchmarks of the Narrow environment in Fig.\ref{fig:envNarrow}
    }
    \label{fig:bench_corridor}
\end{figure*}
\begin{figure*}
    \centering
        \includegraphics[width=0.85\textwidth]{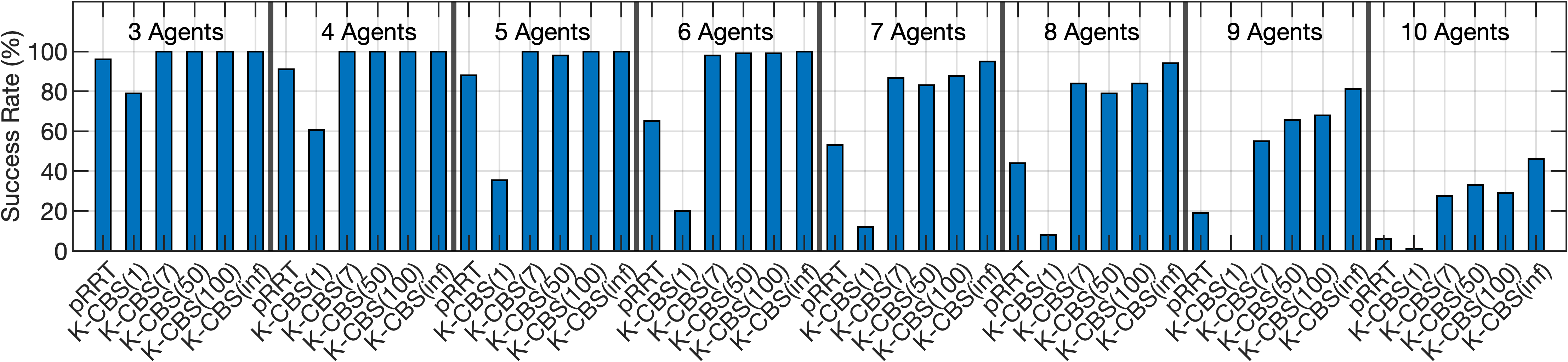}%
    
        \includegraphics[width=.85\textwidth]{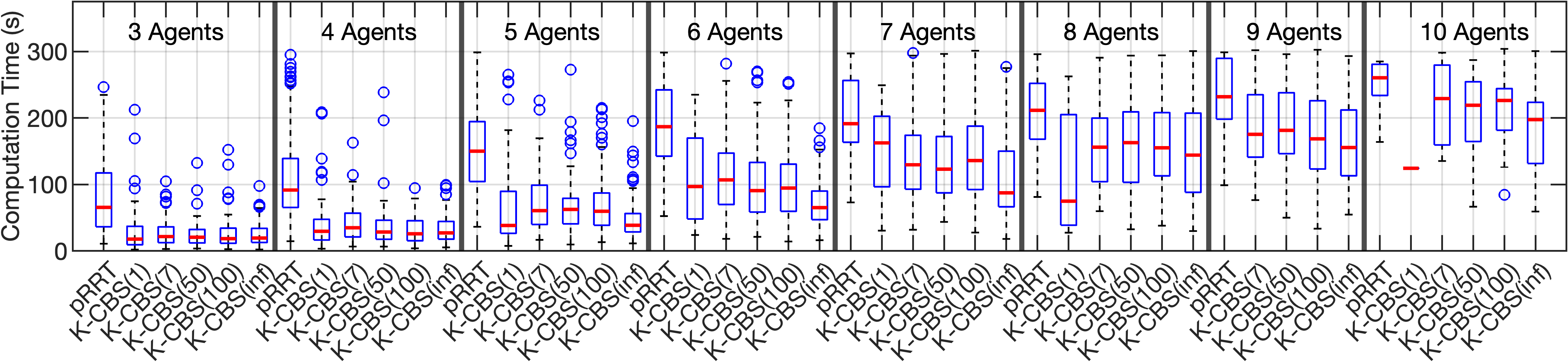}%
    \caption{Benchmark results for the Cluttered Env. in Fig.~\ref{fig:envCluttered}.}
    \label{fig:bench_cluttered}
    \vspace{-1mm}
\end{figure*}

\paragraph*{Open Environment}
In this environment, the robots have to swap positions, i.e., each robot starts in the center of the goal region of another robot (Fig.\ref{fig:envOpen}). We varied the number of robots from 2 to 6.  The results are shown in Fig.~\ref{fig:bench_open}. 
Note that \Alg{cRRT} is only able to find solutions for the two-robot case (10-dimensional composed state space) with a low success rate.  For larger number of robots, its success rate was zero (hence, not shown in Fig.~\ref{fig:bench_open}).  On the other hand \Alg{pRRT}, due to its decentralized nature, performs well in this environment with a success rate of 100\% for all cases.
However, as the number of robots increases, \Alg{pRRT} spends an increasingly large amount of time collision checking against prior robots during planning, resulting in long computation times.  \Alg{K-CBS} does not consider other agents until an entire plan is found. Thus, it scales much better than \Alg{pRRT} with 1--2 orders of magnitude smaller computation times, especially if $B$ is large. 
Observe that for $B = 1$, \Alg{K-CBS} merges robots after one conflict, leading to planning in a larger dimensional space.  Clearly, in open environments such a hasty merge is inadvisable.

\paragraph*{Narrow Environment}
In this environment, the narrow corridors are only wide enough for a single robot to traverse at a given time (Fig.~\ref{fig:envNarrow}). Thus, the blue robot must remain out of the corridor until the green agent leaves it, at which point it may travel safely to goal. The benchmark results are shown in Fig.~\ref{fig:bench_corridor}. \Alg{pRRT} struggles in this space because there is no way to consider the blue and green agents simultaneously to ensure success. On the other hand, \Alg{cRRT}, which considers all three agents simultaneously, always fails due to the blow-up in the (15-dimensional) search space.  Conceptually, \Alg{K-CBS} 
finds a balance in the middle
by merging the blue and green robots but not the red one. 
This leads to a high success rate in this space, 
particularly for low values of $B$, where the composed agents have the most time to traverse the narrow corridor.

\paragraph*{Cluttered Environment}
Here, we further consider the scalability of \Alg{K-CBS} against \Alg{pRRT} by analyzing it on a larger space with many obstacles and increasing the number of robots from 3 to 10 (Fig.~\ref{fig:envCluttered}). The results are shown in Fig.~\ref{fig:bench_cluttered}. Once again, we see that \Alg{pRRT} spends most of its time collision checking and does not scale past $7$ robots. For \Alg{K-CBS}, higher values of $B$ are more successful in all cases while also generating solutions faster than \Alg{pRRT}. 

\paragraph*{Large Environment}
The large environment, depicted in Fig.~\ref{fig:envLarge}, is used to test the limits of \Alg{K-CBS}. The results are in Table~\ref{tab:benchmarkTable}. \Alg{pRRT} failed on every instance, but \Alg{K-CBS} reliably scales up to $20$ agents. 

Overall, our benchmarks of \Alg{K-CBS} for many $B$-values and scenarios suggest that (i) \Alg{K-CBS} is scalable and (ii) merging is most useful in cases where agents must carefully coordinate their movement together in a small space (e.g., corridor).  In more open spaces, large values should be chosen for $B$ to avoid unnecessary merges.
Still, \Alg{K-CBS} solves very complex MRMP instances reliably and quickly compared to current methods.

\begin{table}[t]
\caption{Benchmark results for \Alg{K-CBS} in variations of the Large env. (Fig.~\ref{fig:envLarge}). \Alg{pRRT} always timed out.}
\centering
\begin{tabular}{| c | c | c | c | c|} 
 \hline
 \# Agents & \# Obstacles &  succ. rate (\%) & ave. comp. time (s) \\ [0.5ex] 
 \hline\hline
  10 & 10 & 92 & 161.7\\ 
 \hline
  15 & 10 & 74 & 226.7\\ 
 \hline
 20 & 10 & 30 & 290.8\\
 \hline
  10 & 5 & 98 & 126.7\\ 
 \hline
  15 & 5 & 92 & 203.1\\ 
 \hline
 20 & 5 & 50 & 279.8 \\
 \hline
  10 & 0 & 100 & 104.8 \\ 
 \hline
  15 & 0 & 94 & 152.6 \\ 
 \hline
 20 & 0 & 74 & 239.8 \\
 \hline
\end{tabular}
\label{tab:benchmarkTable}
\end{table}



    
    
    

\section{Concluding Remarks}
\label{sec:Conclusion}

In this work, we introduced an MRMP algorithm that is both scalable and probabilitically complete and poses no simplifying assumptions on the problem.  It initially operates as a fully-decentralized method by individually planning for each robot and posing constraints as conflicts between the robots are encountered.  It is equipped with a merging procedure that allows it to merge the robots with entangled paths to a meta-robot.  By doing so, it is able to utilize full-information (centralized) planning for only those robots while still treating others individually, leading to strong performance in various environments and settings.  
One can extend this work in several directions, such as designing heuristics to further improve performance and an adaptive method of choosing a value for the merging parameter.

\bibliographystyle{IEEEtran}
\bibliography{main}
\end{document}